\documentclass[10pt,twocolumn,letterpaper]{article}

\usepackage{wacv}
\usepackage{times}
\usepackage{epsfig}
\usepackage{graphicx}
\usepackage{amsmath}
\usepackage{amssymb}
\usepackage{booktabs}
\usepackage{bm}
\usepackage{bbm}
\usepackage{amsthm}
\usepackage{mathptmx}
\usepackage{listings}
\usepackage{xcolor}
\usepackage{url}

\definecolor{codegreen}{rgb}{0,0.6,0}
\definecolor{codegray}{rgb}{0.5,0.5,0.5}
\definecolor{codepurple}{rgb}{0.58,0,0.82}
\definecolor{backcolour}{rgb}{0.95,0.95,0.92}

\lstdefinestyle{mystyle}{
    commentstyle=\color{codegreen},
    keywordstyle=\color{magenta},
    numberstyle=\tiny\color{codegray},
    stringstyle=\color{codepurple},
    basicstyle=\ttfamily\footnotesize,
    breakatwhitespace=false,         
    breaklines=true,                 
    captionpos=b,                    
    keepspaces=true,                 
    numbers=left,                    
    numbersep=5pt,                  
    showspaces=false,                
    showstringspaces=false,
    showtabs=false,                  
    tabsize=2
}

\newtheorem{theorem}{Theorem}

\newtheorem{definition}{Definition}
\newtheorem{corollary}{Corollary}

\wacvfinalcopy % *** Uncomment this line for the final submission

\ifwacvfinal
\def\assignedStartPage{1} % *** Enter the assigned starting page number (instead of 9876)
\fi

\ifwacvfinal
\usepackage[breaklinks=true,bookmarks=false]{hyperref}
\else
\usepackage[pagebackref=true,breaklinks=true,colorlinks,bookmarks=false]{hyperref}
\fi

\ifwacvfinal
\else
\fi

\begin{document}

%%%%%%%%% TITLE
\title{Density-Fixing: Simple yet Effective Regularization Method \\ based on the Class Priors}

\author{Masanari Kimura\\
Ridge-i Inc.\\
{\tt\small mkimura@ridge-i.com}
\and
Ryohei Izawa\\
Ridge-i Inc.\\
{\tt\small rizawa@ridge-i.com}
}

\maketitle
%\thispagestyle{empty}

%%%%%%%%% ABSTRACT
\begin{abstract}
Machine learning models suffer from overfitting, which is caused by a lack of labeled data.
To tackle this problem, we proposed a framework of regularization methods, called density-fixing, that can be used commonly for supervised and semi-supervised learning.
Our proposed regularization method improves the generalization performance by forcing the model to approximate the class's prior distribution or the frequency of occurrence.
This regularization term is naturally derived from the formula of maximum likelihood estimation and is theoretically justified.
%We further investigated the asymptotic behavior of the proposed method and how the regularization terms behave when assuming a prior distribution of several classes in practice.
We further provide the several theoretical analyses of the proposed method including asymptotic behavior.
Our experimental results on multiple benchmark datasets are sufficient to support our argument, and we suggest that this simple and effective regularization method is useful in real-world machine learning problems.
\end{abstract}

%%%%%%%%% BODY TEXT
\section{Introduction}
Machine learning has achieved great success in many areas.
However, such machine learning models suffer from an over-fitting problem caused by a lack of data~\cite{hawkins2004problem,lawrence1997lessons,dietterich1995overfitting}.
To tackle such problems, research on semi-supervised learning~\cite{zhu2009introduction,kingma2014semi} or regularization~\cite{zhang2018mixup,srivastava2014dropout} has been very active.
%The main idea of semi-supervised learning is to solve supervised learning problems with few labels by utilizing unlabeled data.
%In real-world machine learning problems, labeled data is often scarce, but unlabeled data is abundant.
%Therefore, semi-supervised learning methods that make good use of unlabeled data are essential.

We focus on leveraging the class density of the entire dataset as prior knowledge about labeled and unlabeled data.
This means that we assume that the density of each class is obtained as prior knowledge.
This assumption is a natural one in many actual machine learning problems.
Based on this idea, we propose a framework of regularization methods, called density-fixing, both supervised and semi-supervised settings can commonly use that.
Our proposed density-fixing regularization improves the generalization performance by forcing the model to approximate the class's prior distribution or the frequency of occurrence.
This regularization term of density-fixing is naturally derived from the formula for maximum likelihood estimation and is theoretically justified.
We further investigated the asymptotic behavior of the density-fixing and how the regularization terms behave when assuming a prior distribution of several classes in practice.
Our experimental results on multiple benchmark datasets are sufficient to support our argument, and we suggest that this simple and effective regularization method is useful in real-world problems.

{\bf Contribution}:
We propose the density-fixing regularization, which has the following properties:
\begin{itemize}
    \item simplicity: density-fixing is very simple to implement and has almost no computational overhead. 
    \item naturalness: density-fixing is derived naturally from the formula for maximum likelihood estimation and has a theoretical guarantee. We also provide several theoretical analyses.
    \item versatility: density-fixing is generally applicable to many problem settings.
\end{itemize}
In a nutshell, density-fixing forcing the balance of class density:
\begin{equation}
    L_\theta(\bm{x}, y) = \ell(\bm{x}, y) + \gamma \cdot D_{KL}[p_\theta(y)\|q_(y)], \label{eq:density_fixing}
\end{equation}
where $\ell(\bm{x}, y)$ is the some loss function (e.g. cross-entropy loss), and $\gamma \geq 0$ is the parameter of the regularization term.
For the true distribution $q(y)$ of a class, we can use it if it is given as prior knowledge, otherwise we can average the frequency of occurrence of the labels in the training sample and use it as an estimator $\hat{q}(y)$:
\begin{equation}
    \hat{q}(y) = \{y^1,\dots,y^K\},\ \  y^i = \frac{1}{N}\sum^N_{j=1}\mathbbm{1}_{\{c(\bm{x}_j)=i\}}.
\end{equation}
The sample mean provides the unbiased and consistent estimator of the frequency of class occurrence, so it is sufficient to use it.

The source-code necessary to replicate our CIFAR-10 experiments is available at GitHub~\footnote{\url{https://github.com/nocotan/density_fixing}}.

%The source-code necessary to replicate our CIFAR-10 experiments will be available at GitHub.

\section{Related Works}
In this section, we introduce some related works that are relevant to our work.
\subsection{Over-fitting and Regularization}
Machine learning models suffer from an over-fitting problem caused by a lack of data.
In order to avoid over-fitting, various regularization methods have been proposed.
For example, Dropout~\cite{srivastava2014dropout} is a powerful regularization method that introduces ensemble learning-like behavior by randomly removing connections between neurons of the Deep Neural Network.
Another recently proposed simple regularization method is mixup and its variants~\cite{zhang2018mixup,kimura2020mixup,yun2019cutmix}, which takes a linear combination of training data as a new input.
There are many regularization methods for some specific models (e.g., for Generative Adversarial Networks~\cite{roth2017stabilizing,kimura2018anomaly}).

\subsection{Semi-Supervised Learning}
There are many studies on semi-supervised learning~\cite{zhu2009introduction,zhu2005semi}.
The method of assigning pseudo-labels to unlabeled data as new training data is very popular~\cite{lee2013pseudo}.
Another approach to semi-supervised learning is the use of Generative Adversarial Networks, which are famous for their expressive power~\cite{kumar2017semi}.

\section{Notations and Problem Formulation}
Let $\mathcal{X}$ be the input space, $\mathcal{Y} = \{1,\dots,K\}$ be the output space, $K$ be the number of classes and $\mathcal{C}$ be a set of concepts we may wish to learn.
We assume that each input vector $\bm{x}\in\mathbb{R}^d$ is of dimension d. We also assume that examples are independently and identically distributed (i.i.d) according to some fixed but unknown distribution $\mathcal{D}$.

Then, the learning problem formulated as follows: we consider a fixed set of possible concepts $\mathcal{H}$, called hypothesis set.
We receives a sample $\mathcal{B} = (\bm{x}_1,\dots,\bm{x}_N)$ drawn i.i.d. according to $\mathcal{D}$ as well as the labels $(c(\bm{x}_1),\dots,c(\bm{x}_N))$, which are based on a specific target concept $c\in\mathcal{C}: \mathcal{X}\mapsto\mathcal{Y}$. 
In the semi-supervised learning problem, we additionally have access to unlabeled sample $\mathcal{B}^U = (\bm{x}'_1,\dots,\bm{x}'_{N_U})$ drawn i.i.d according to $\mathcal{D}$.
Our task is to use the labeled sample $\mathcal{B}$ and unlabeled sample $\mathcal{B}^U$ to find a hypothesis $h^*\in\mathcal{H}$ that has a small generalization error for the concept $c$.
The generalization error $\mathcal{R}$ is defined as follows.
\begin{definition}{(Generalization error)}
\label{def:generalization_error}
Given a hypothesis $h\in{\mathcal{H}}$, a target concept $c\in\mathcal{C}$, and unknown distribution $\mathcal{D}$, the generalization error of $h$ is defined by
\begin{equation}
    \mathcal{R}(h) = \mathbb{E}_{x\sim{D}}\Big[\mathbbm{1}_{h(\bm{x})\neq{c(\bm{x}})}\Big],
\end{equation}
where $\mathbbm{1}_\omega$ is the indicator function of the event $\omega$.
\end{definition}
The generalization error of a hypothesis $h$ is not directly accessible since both the underlying distribution $D$ and the target concept $c$ are unknown
Then, we have to measure the empirical error of hypothesis $h$ on the observable labeled sample $\mathcal{B}$. The empirical error $\hat{\mathcal{R}}(h)$ is defined as follows.
\begin{definition}{(Empirical error)}
\label{def:empirical_error}
Given a hypothesis $h\in{H}$, a target concept $c\in\mathcal{C}$, and a sample $\mathcal{B} = (\bm{x}_1,\dots,\bm{x}_N)$, the empirical error of $h$ is defined by
\begin{equation}
    \hat{\mathcal{R}}(h) = \frac{1}{n}\sum^n_{i=1}\mathbbm{1}_{h(\bm{x}_i)\neq{c(\bm{x}_i)}}.
\end{equation}
\end{definition}

In learning problems, we are interested in how much difference there is between empirical and generalization errors.
Therefore, in general, we consider the relative generalization error $\hat{\mathcal{R}}(h)-\mathcal{R}(h)$.
In the following sections, we derive a regularization method that reduces the relative generalization error $\hat{\mathcal{R}}(h)-\mathcal{R}(h)$ and asymptotic behavior of our proposed method.

\section{Density-Fixing Regularization}
In this paper, we assume that $\mathcal{H}_p$ is a class of functions mapping input vectors to the class densities:
\begin{equation}
    h(\bm{x}) = \Big\{\bm{x}\mapsto p(y|\bm{x}) \Big\},
\end{equation}
Therefore, we can replace the learning problem with a problem that approximates the true distribution $q(y|\bm{x})$ with the estimated distribution $p(y|\bm{x})$.

We assume that the class-conditional probability for labeled data $p(\bm{x}|y)$ and that for unlabeled data (or test data) $q(\bm{x}|y)$ are the same:
\begin{equation}
    p(\bm{x}|y) = q(\bm{x}|y). \label{eq:class_conditional}
\end{equation}
Then, our goal is to estimate $q(y|\bm{x})$ from labeld data $\{\bm{x}_i, y_i\}^N_{i=1}$ drawn i.i.d from $p(\bm{x}, y)$ and unlabeled data $\{\bm{x}'_i\}^{N_U}_{i=1}$

\begin{theorem}
Let $p_\theta(y|\bm{x})$ be the estimated distribution parameterized by $\theta$, and $q(y|\bm{x})$ be the true distribution.
Then, we can write the sum of log-likelihood function as follows:
\begin{equation}
    \sum\log{L(\theta)} = \sum\log{p_\theta(y|\bm{x})} - D_{KL}[p_\theta(y)\| q(y)], \label{eq:kld_reg}
\end{equation}
where $D_{KL}[P\| Q]$ is the Kullback-Leibler divergence~\cite{kullback1951on} from $Q$ to $P$:
\begin{eqnarray}
    D_{KL}[P\|Q] &=& \sum P(x)\log{\frac{P(x)}{Q(x)}} \\
    &=& -\sum P(x)\log{\frac{Q(x)}{P(x)}}
\end{eqnarray}
\end{theorem}
This means that when we consider maximum likelihood estimation, we can decompose the objective function into two terms: the term depending on $\bm{x}$ and the term depending only on $y$.

\begin{proof}
From Bayes' theorem, we can obtain
\begin{eqnarray}
p(\bm{x}|y) &=& \frac{p(y|\bm{x})p(\bm{x})}{p(y)}, \label{eq:bayes1} \\
q(\bm{x}|y) &=& \frac{q(y|\bm{x})q(\bm{x})}{q(y)}. \label{eq:bayes2}
\end{eqnarray}
Then, combining Eq~\eqref{eq:class_conditional}, \eqref{eq:bayes1} and \eqref{eq:bayes2},
\begin{eqnarray}
\frac{p(y|\bm{x})p(\bm{x})}{p(y)} &=& \frac{q(y|\bm{x})q(\bm{x})}{q(y)} \nonumber\\
q(y|\bm{x}) &=& \frac{q(y)}{p(y)}\frac{p(\bm{x})}{q(\bm{x})}p(y|\bm{x}) \nonumber\\
&\propto& \frac{q(y)}{p(y)}p(y|\bm{x}).
\end{eqnarray}
The maximum likelihood estimator is the maximizer of the likelihood:
\begin{eqnarray}
    \hat{\theta} &=& argmax_{\theta}\prod^N_{i=1}\frac{q(y)}{p_\theta(y)}p_\theta(y|\bm{x}) \nonumber \\
    &=& argmax_{\theta}\sum^N_{i=1}\log{\Big\{\frac{q(y)}{p_\theta(y)}p_\theta(y|\bm{x})\Big\}}.
\end{eqnarray}
Considering the maximization of the likelihood function, we can have the log-likelihood function $\log{L(\theta)}$ as follows:
\begin{eqnarray}
\log{L(\theta)} &=& \log{\Biggl\{\frac{q(y)}{p_\theta(y)}p_\theta(y|\bm{x})\Biggr\}} \nonumber\\
&=& \log{p_\theta(y|\bm{x})} + \log{\frac{q(y)}{p_\theta(y)}}. \label{eq:mle_with_class_density}
\end{eqnarray}
Finally, we compute sum of log-likelihood function,
\begin{eqnarray}
\sum\log{L(\theta)} &\simeq& \sum\log{p_\theta(y|\bm{x})} + \mathbb{E}\Biggl[\log{\frac{q(y)}{p_\theta(y)}}\Biggr] \nonumber\\
&=& \sum\log{p_\theta(y|\bm{x})} - D_{KL}[p_\theta(y)\|q(y)]. \nonumber\\
\end{eqnarray}
and then, we have Eq~\eqref{eq:kld_reg}.
\end{proof}
Considering that we maximize Eq~\eqref{eq:kld_reg}, it is clear that $D_{KL}[p_\theta(y)\|q(y)]$ should be small value or closer to $0$.
The Kullback–Leibler divergence is defined if and only if $\forall{y}$, $q(y)=0$ implies $p_\theta(y)=0$, and this property is so called absolute continuity.

From the above theorem, if the probability of class occurrence is known in advance, it can be used to perform regularization.
We call this term density-fixing regularization.
Regularization is performed so that the density of each class in the inference result for the unlabeled sample $p_\theta(y)$ approximates the $q(y)$.

In addition, the well-known properties of KL-divergence also naturally lead to the following proposition about density-fixing.
\begin{corollary}{(Zero-forcing property)}
\label{cor:zero-forcing}
The estimation $p_\theta$ that approximates $q$ satisfies $p_\theta(y) = 0 \Rightarrow q(y) = 0$.
\end{corollary}
\begin{proof}
If we approximate $p_\theta(y)$ and $q(y)$ by continuous functions, then the KL-divergence is given as follows:
\begin{equation*}
    D_{KL} = \int p_\theta(y)\log{\frac{p_\theta(y)}{q(y)}}.
\end{equation*}
Since $p_\theta$ and $q$ are continuous functions and then, if $q(y) = 0$ at some point $y_0$, there exists a neighborhood around $y_0$, where $q(y)<\epsilon$.
Therefore, if $p_\theta(y)>0$, $p_\theta(y)/q(y)$ will be very large in that neighborhood, which makes the integral grow very large.
Thus, any minimization algorithm will result in a $p_\theta(y)$, where $p_\theta(y) = 0$ if $q(y) = 0$. 
\end{proof}
This mean that the best approximation $\hat{p}_\theta$ satisfies
\begin{equation}
    \hat{p}_\theta(y) = 0,
\end{equation}
for $y$ at which $q(y)=0$. This property is called zero-forcing, and we can see that our regularization behave as if the probabilities of classes we do not know remain $0$.

\begin{figure}[t]
    \centering
    \includegraphics[scale=0.3, bb=100 0 749 377]{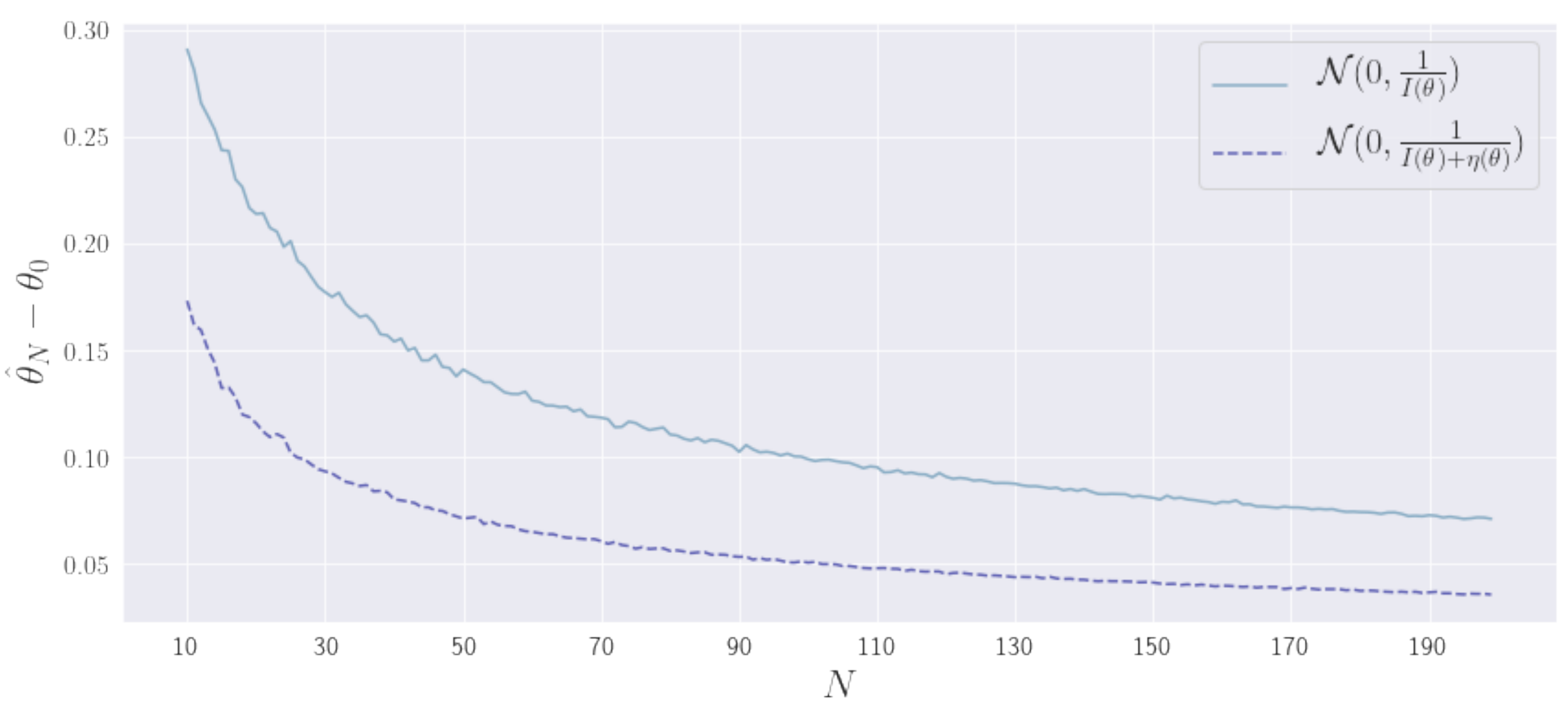}
    \caption{\label{fig:asymptotic_variance} Difference in the behavior of the asymptotic variance of the maximum likelihood estimator by density-fixing regularization.
    We can see that our regularization improves the error of the estimators quickly.}
\end{figure}

\section{Asymptotic Normality}
In this section, we discuss how the density-fixing regularization behaves asymptotically.
In general, the bias of the estimator $\hat{\theta}$ is defined by
\begin{equation}
    b(\hat{\theta}) = \mathbb{E}[\hat{\theta}] - \theta, 
\end{equation}
when $\theta$ is the true parameter.
An estimator is unbiased when $b(\hat{\theta})=0$.

The asymptotic theory studies the behavior of an estimator when sample size $N$ is large.
When the bias satisfies
\begin{equation}
    \lim_{N\to\infty}b(\hat{\theta}) = 0,
\end{equation}
this estimator is asymptotically unbiased.
It is known that the maximum likelihood estimator is asymptotically unbiased.
This means that the maximum likelihood estimator converges to the true parameter as $N$ tends to infinity:
\begin{equation}
    \lim_{N\to\infty}\hat{\theta} = \theta.
\end{equation}
Then, the accuracy of an estimator is measured by the error covariance matrix, $\bm{V} = (V_{ij})$,
\begin{equation}
    V_{ij} = \mathbb{E}\Big[(\hat{\theta}_i - \theta_i)(\hat{\theta}_j - \theta_j)\Big].
\end{equation}
The following theorem states that density-fixing reduces the error covariance of the maximum likelihood estimator depending on the shape of the distribution.

\begin{theorem}\label{thm:asym_reg}
Let $\ell(\theta) = \log L(\theta)$.
The asymptotic variance of the maximum likelihood estimator applying the density-fixing regularization is given by $\frac{1}{I(\theta) + \eta(\theta)}$. Here, $\eta(\theta)$ is a function that always takes a positive value, parameterized by $\theta$.

\end{theorem}
\begin{proof}
In the maximum likelihood estimator $\hat{\theta}_N$ for the number of samples $N$, we can obtain the following by Taylor expansion of $\ell(\theta)$ around $\theta_0$:
\begin{eqnarray}
0 &=& \frac{\partial\ell(\theta_0)}{\partial\theta} + \frac{\partial^2\ell(\theta_0)}{\partial\theta^2}(\hat{\theta}_N-\theta_0) \nonumber\\
&& + \frac{\partial^3\ell(\theta_0)}{\partial\theta^3}\frac{(\hat{\theta}_N-\theta_0)^2}{2!} \nonumber \\
\hat{\theta}_N - \theta_0 &=& - \frac{2\ell'(\theta_0)}{2\ell''(\theta_0)+\ell'''(\theta_0)(\hat{\theta}_N - \theta_0)} \nonumber \\
\sqrt{N}(\hat{\theta}_N - \theta_0) &=& - \frac{2\ell'(\theta_0)/\sqrt{N}}{2\ell''(\theta_0)/N + \ell'''(\theta_0)(\hat{\theta}_N - \theta_0)/N}, \label{eq:mle_asym1} \nonumber\\
\end{eqnarray}
here, we assume that $\ell(\theta)$ has the third-order derivative with respect to parameter $\theta$ and $\frac{\partial^3}{\partial\theta^3}\ell(\theta)$ be bounded.
From Eq~\eqref{eq:mle_asym1} and central limit theorem, we can obtain
\begin{eqnarray}
\sqrt{N}(\hat{\theta}_N-\theta_0) \sim N\Big(0, \frac{1}{I(\theta)}\Big),
\end{eqnarray}
when $N$ is sufficiently large.
Here, $I(\theta)$ is the Fisher information matrix:
\begin{equation}
    I(\theta) = -\mathbb{E}\Big[\frac{\partial^2}{\partial\theta^2}\ell(\theta)\Big].
\end{equation}
Then, let $\log{p_\theta(y|\bm{x})}=f(\theta)$ as the original likelihood function, we can obtain
\begin{eqnarray}
&& \frac{\partial^2}{\partial\theta^2}\Biggl\{\log{f(\theta)} + \mathbb{E}\Big[\log{\frac{q(y)}{p_\theta(y)}}\Big]\Biggr\} \nonumber \\
&=& \frac{\partial^2}{\partial\theta^2}f(\theta) + \frac{\partial^2}{\partial\theta^2}\Biggl\{\mathbb{E}\Big[\log{\frac{q(y)}{p_\theta(y)}}\Big]\Biggr\} \nonumber \\
&=& \frac{\partial^2}{\partial\theta^2}f(\theta) + \frac{\partial^2}{\partial\theta^2}\mathbb{E}\Big[\log{q(y)} - \frac{\partial^2}{\partial\theta^2}\log{p_\theta(y)}\Big] \nonumber \\
&=& \frac{\partial^2}{\partial\theta^2}f(\theta) - \frac{\partial^2}{\partial\theta^2}\mathbb{E}\Big[\log{p_\theta(y)}\Big].
\end{eqnarray}
Therefore, the maximum likelihood estimator applying the density-fixing regularization $\hat{\theta}^*_N$ satisfies the following:
\begin{equation}
\sqrt{N}(\hat{\theta}^*_N-\theta_0) \sim N\Biggl(0, \frac{1}{I(\theta) - \frac{\partial^2}{\partial\theta^2}\mathbb{E}[\log{p_\theta(y)}]}\Biggr).
\end{equation}

Here, let $p_\theta(y)$ be the exponential family.
An exponential family of probability distribution is written as
\begin{equation}
    p(y, \bm{\theta}) = \exp{\Big\{\zeta(\bm{\theta})T(\bm{y}) + k(\bm{y}) - \psi(\bm{\theta})\Big\}},
\end{equation}
where $\zeta(\bm{\theta})$ is the natural parameter and $T(\bm{y})$ is the sufficient statistcs of $y$.
Let $h(\bm{y}) = \exp{(k(\bm{y}))}$ and $g(\bm{\theta}) = \exp{(-\psi(\bm{\theta}))}$, we can have
\begin{align}
    p(y, \bm{\theta}) &= h(y)g(\bm{\theta})\exp{\Big\{\zeta(\bm{\theta})T(y)\Big\}}, \\
    \log{p(y, \bm{\theta})} &= \log{(h(y))} + \log{(g(\bm{\theta}))} + \zeta(\bm{\theta})T(y). \nonumber\\
\end{align}
Here, expressing $\psi(\bm{\theta})$ by $\psi(\bm{\zeta})$ through variable transformation, we obtain the following equation:
\begin{align}
    p(\bm{y}, \bm{\zeta}) &= \Biggl(\prod_{i=1}h(y_i)\Biggr)(g(\zeta))^n\exp{\Biggl(\sum_{i=1}\bm{\zeta}T(y_i)\Biggr)} \nonumber \\
     &= \Biggl(\prod_{i=1}h(y_i)\Biggr)(g(\zeta))^n\exp{\Biggl(N\cdot\bm{\zeta}\sum_{i=1}T(y_i)\Biggr)} \nonumber \\
     \log{p(\bm{y}, \bm{\zeta})} &= \log{\prod_{i=1} h(y_i)} + N\Biggl\{\log{(g(\bm{\zeta}))} + \bm{\zeta}\sum_{i=1}T(y_i)\Biggr\}. \nonumber \\
\end{align}
Then, the second-order derivative of $\log{p(\bm{y}) by \bm{\zeta}}$ is
\begin{align}
    \frac{\partial}{\partial\bm{\zeta}} &= N\frac{g'(\bm{\zeta})}{g(\bm{\zeta})} + N\cdot T(\bm{y}) \nonumber \\
    \frac{\partial}{\partial\zeta^2} &= \frac{\partial}{\partial\zeta}\Big(N\frac{g'(\bm{\zeta})}{g(\bm{\zeta})} + N\cdot T(\bm{y})\Big) \nonumber \\
    &= N\Big(\frac{g''(\bm{\zeta})}{g(\bm{\zeta})} - \Big(\frac{g'(\bm{\zeta})}{g(\bm{\zeta})}\Big)^2\Big) 
    = -N\cdot Var(T(\bm{y})).
\end{align}
This means that the second-order derivative of the log-likelihood function of the exponential family has always negative value.
Therefore, we can obtain the proof of Theorem~\ref{thm:asym_reg} with $\eta(\theta) = -\frac{\partial}{\partial\theta^2}\mathbb{E}[\log{p_\theta(y)}]$.
\end{proof}
This theorem implies that the convergence rate of the asymptotic variance of the maximum likelihood estimator becomes faster by $\eta(\theta)$ by applying the density-fixing regularization.
Figure~\ref{fig:asymptotic_variance} illustrates the asymptotic behavior of the estimator by our regularization.

\begin{figure}[t]
    \centering
    \includegraphics[scale=0.28, bb=0 0 954 352]{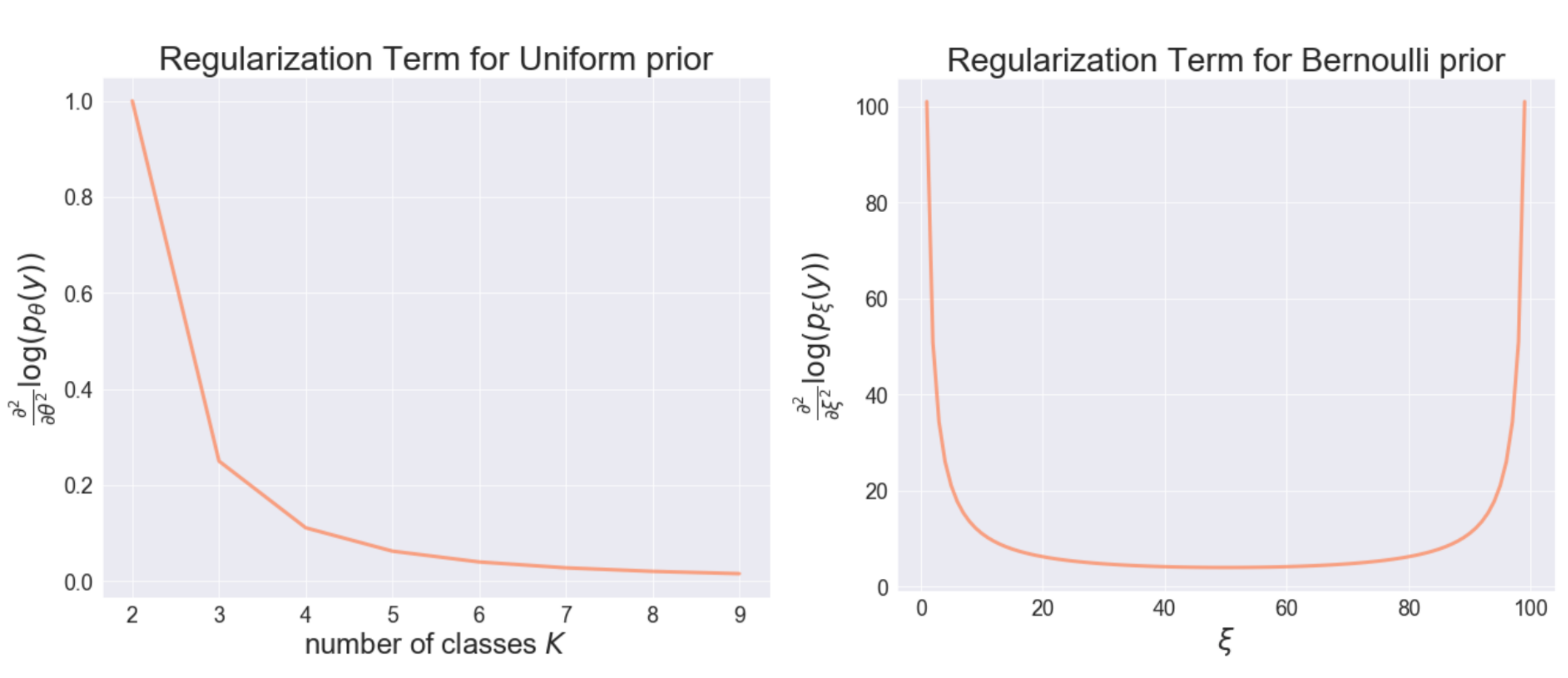}
    \caption{\label{fig:regularization_term}Behavior of the regularization term on each parameter. Left: the regularization term when $p(y)$ is a discrete uniform distribution.
    Right: the regularization term when $p(y)$ is a Bernoulli distribution.}
\end{figure}

\subsection{Some Examples}
In this subsection, we investigate the behavior of our proposed method by assuming some class distributions as examples.
To summarize our results:
\begin{itemize}
    \item For discrete uniform distribution, the effect of regularization becomes weaker as the number of classes increases,
    \item For Bernoulli distribution, our regularization behaves to give strong regularization when there is a class imbalance.
\end{itemize}
Figure~\ref{fig:regularization_term} shows the behavior of the regularization terms under each distribution.

\subsubsection{Discrete Uniform Distribution}
We assume that the probability density function of classes $p(y)$ is as follows:
\begin{equation}
    p(y) = \frac{1}{K-1},
\end{equation}
here $K$ is the number of classes.
This is the discrete uniform distribution $U(1, K)$.
From theorem~\ref{thm:asym_reg}, the following corollary can be derived.
\begin{corollary}
The effect of regularization becomes weaker as the number of classes increase when the class probability follows a discrete uniform distribution.
\end{corollary}
\begin{proof}
From theorem~\ref{thm:asym_reg}, the asymptotic behavior of density-fixing training is
\begin{eqnarray}
\sqrt{N}(\hat{\theta}^*_N - \theta)\sim N\Biggl(0, \frac{1}{I(\theta) + 1/(K-1)^2}\Biggr). \label{eq:uniform}
\end{eqnarray}
It is clear that $K\propto^{-1}\text{regularization effect}$ from the above equation.
\end{proof}
The discrete uniform distribution corresponds to the case of multi-class classification where all classes have the same probability of occurrence and is a general problem setup.

\begin{figure*}[t]
    \centering
    \includegraphics[scale=0.6, bb=0 0 1746 243]{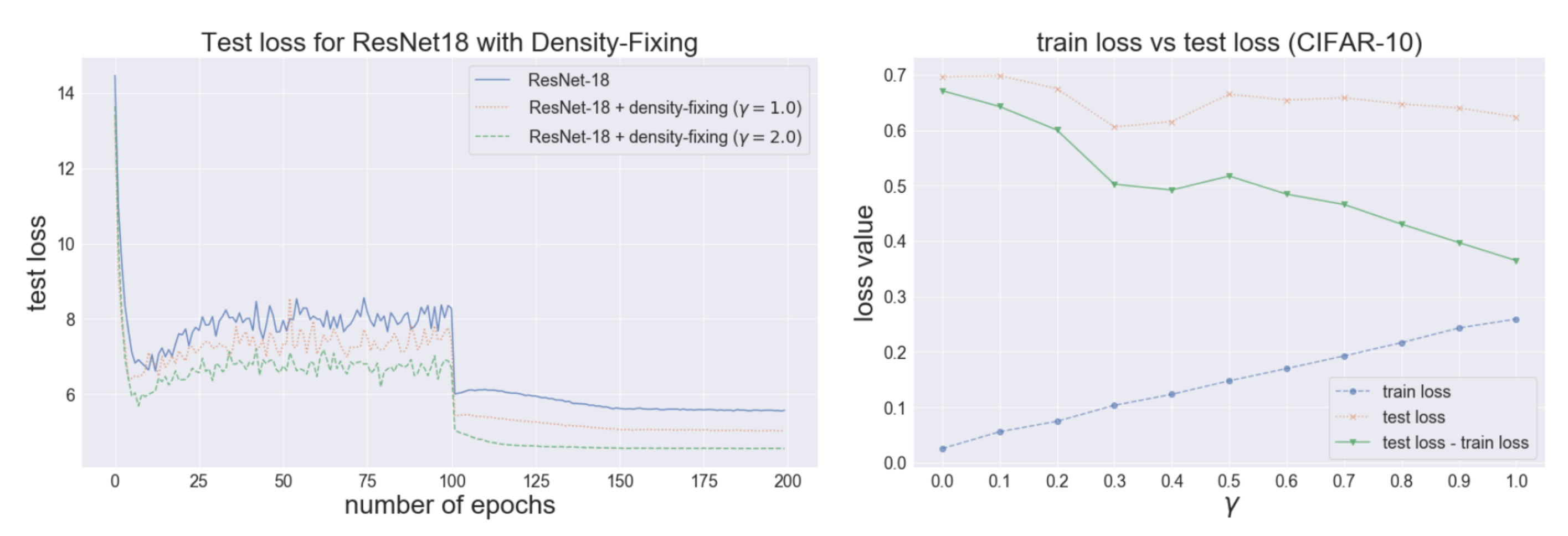}
    \caption{\label{fig:supervised_train_vs_test} Left: Test error evolution for the best baseline model and density-fixing.
    We can see that the test loss can be reduced by density-fixing regularization.
    Right: Test loss and train-test differences for each $\gamma$ in the supervised setting.
    It is noteworthy that the gap between train loss and test loss becomes smaller as the value of the density-fixing regularization parameter is increased.}
\end{figure*}

\subsubsection{Bernoulli Distribution}
We assume that $K=2$ and the probability density function of classes $p(y)$ is as follows:
\begin{equation}
    p(y) = \begin{cases}
    \xi & y = 1, \\
    (1-\xi) & y = 0,
    \end{cases}
\end{equation}
here $\xi\in[0, 1]$ and this is the Bernoulli distribution.
This case is a binary classification problem.
\begin{corollary}
For binary classification problem, our proposed method behaves to give strong regularization when there is a class imbalance.
\end{corollary}
\begin{proof}
Our regularization term is
\begin{eqnarray}
-\frac{\partial^2}{\partial\xi^2}\Biggl\{\mathbb{E}\Big[\log{p(y)}\Big]\Biggr\} &=& -\mathbb{E}\Biggl[\frac{\partial^2}{\partial\xi^2}\log{p(y)}\Biggr] \nonumber\\
&=& I(\xi) = \frac{1}{\xi(1-\xi)},
\end{eqnarray}
and we can obtain the asymptotic behavior of density-fixing training as follows:
\begin{eqnarray}
\sqrt{N}(\hat{\theta}^*_N - \theta)\sim N\Biggl(0, \frac{1}{I(\theta) + 1/\xi(1-\xi)}\Biggr).
\end{eqnarray}
Thus, we can see that regularization is stronger when $\xi$ is away from $1/2$.
\end{proof}

\section{Connection to Knowledge Distillation}
Knowledge distillation~\cite{44873} is the model compression method in which a small model is trained to mimic a pre-trained larger model.
The large pretrained model is called the teacher model and the small model learned by knowledge distillation is called the student model.
In knowledge distillation, knowledge is transferred from the teacher model to the student model by minimizing a loss function in which the target is the distribution of class probabilities predicted by the teacher model.
Furthermore, it has been shown that knowledge distillation can be used not only for model compression, but also for model regularization~\cite{furlanello2018born,xie2020self}.
For example, it has been experimentally shown that an iterative knowledge distillation method using the same architecture for teacher and student models contributes to the generalization performance of the models~\cite{furlanello2018born}.

Let the output of the teacher model be $p_t(y)$ and the output of the student model be $p_s(y)$.
Here, the knowledge distillation scheme is formulated as follows:
\begin{equation}
    L_\theta = \alpha \ell(\bm{x}, y) + (1 - \alpha)\cdot D_{KL}[p_s(y|\bm{x})\|p_t(y|\bm{x})], \label{eq:knowledge_distillation}
\end{equation}
where $\alpha\in[0, 1]$ is the parameter.
Comparing Eq~\eqref{eq:density_fixing} and Eq~\eqref{eq:knowledge_distillation}, we can see that the knowledge distillation loss function is a special case of the density-fixing loss function.
In other words, knowledge distillation is the injection of knowledge extracted by the teacher model instead of human prior knowledge in density-fixing regularization.
Thus, we can switch between density-fixing and knowledge distillation based on how much we know about the distribution of the data.
Our proposed density-fixing is useful if we are very familiar with the distribution of the data, while knowledge distillation is useful if we have no knowledge of the data at all.

\section{Experimental Results}
In this section, we introduce our experimental results.
We implement the density-fixing regularization as follows:
\begin{equation}
    L(\bm{x}, y) = L_{CE}(\bm{x}, y) + \gamma D_{KL}[p_\theta(y)\|p(y)],
\end{equation}
where $L_{CE}(\bm{x}, y)$ is the cross-entropy loss and $\gamma$ is the weight parameter for the regularization term.
The implementation of density-fixing regularization is straightforward, Figure~\ref{fig:impl} shows the few lines of code necessary to implement density-fixing regularization in PyTorch~\cite{paszke2019pytorch}.

The datasets we use are CIFAR-10~\cite{krizhevsky2009learning}, CIFAR-100~\cite{krizhevsky2009learning}, STL-10~\cite{coates2011analysis} and SVHN~\cite{netzer2011reading}.
We determined the prior distribution of classes based on the number of data accounted for in each class of the data set, and we used ResNet-18~\cite{he2016deep} as the baseline model.

\begin{table}[t]
\centering
\caption{Top 1 and Top 5 test error comparison in the supervised setting.}
\label{tab:experiments_error_supervised}
\scalebox{0.75}{
\begin{tabular}{@{}llll@{}}
\toprule
Dataset & Model                                      & Top 1 Error     & Top 5 Error      \\ \midrule
CIFAR-10   & ResNet-18                               & 12.720\%        &   0.812\%        \\
           & ResNet-18 + density-fixing ($\gamma=1$) & 12.230\%        &   0.779\%        \\
           & ResNet-18 + density-fixing ($\gamma=2$) & ${\bf 12.020\%}$&   ${\bf 0.752\%}$\\
           & ResNet-18 + density-fixing ($\gamma=4$) & 12.341\%        &   0.770\%        \\ \midrule
CIFAR-100  & ResNet-18                               & 25.562\%        &   6.710\%        \\
           & ResNet-18 + density-fixing ($\gamma=1$) & 25.241\%        &   6.302\%        \\
           & ResNet-18 + density-fixing ($\gamma=2$) & 25.965\%        &   6.887\%        \\
           & ResNet-18 + density-fixing ($\gamma=4$) & ${\bf 24.981\%}$&   ${\bf 6.122\%}$\\
            
\bottomrule
\end{tabular}
}
\end{table}

\begin{table}[t]
\centering
\caption{Top 1 test error comparison for each dataset in the semi-supervised setting.
The datasets we use are CIFAR-10, CIFAR-100, STL-10 and SVHN.}
\label{tab:experiments_error}
\scalebox{0.788}{
\begin{tabular}{@{}lllllll@{}}
\toprule
dataset   & $\gamma=0$ & $\gamma=0.1$    & $\gamma=0.2$   & $\gamma=0.4$   & $\gamma=0.8$ & $\gamma=1.0$ \\ \midrule
CIFAR-10  & 28.235     & ${\bf 28.141}$  & 28.510         & 29.086         & 30.964       & 30.892       \\
CIFAR-100 & 66.622     & 66.723          & 66.861         & ${\bf 65.918}$ & 66.895       & 67.007       \\
STL-10    & 59.770     & ${\bf 59.553}$  & 60.110         & 60.124         & 60.405       & 60.897       \\
SVHN      & 27.937     & 28.028          & ${\bf 27.601}$ & 30.110         & 32.025       & 32.879       \\ \bottomrule
\end{tabular}
}
\end{table}

\begin{figure*}[t]
    \centering
    \includegraphics[scale=0.28, bb=0 0 1746 953]{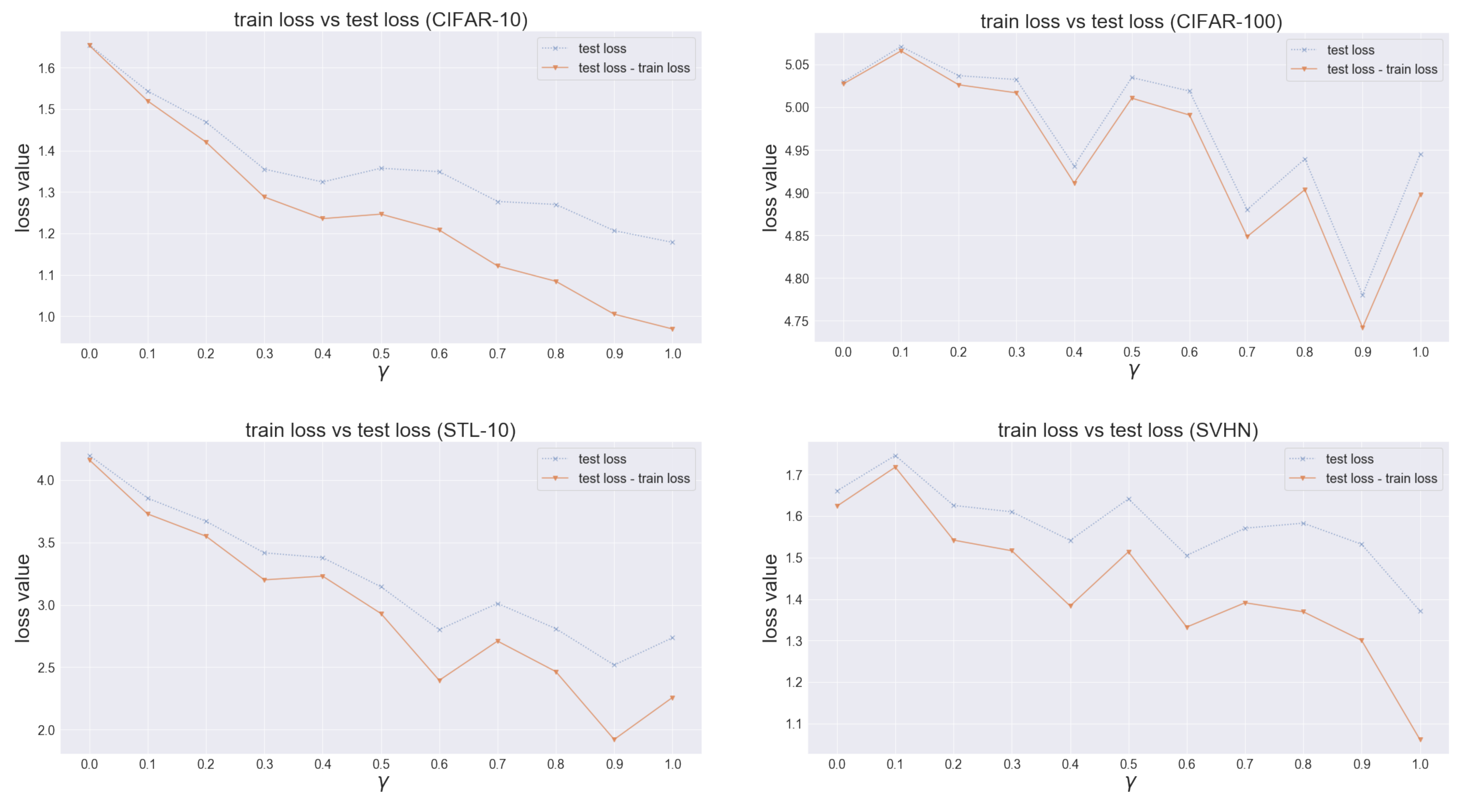}
    \caption{\label{fig:train_vs_test} Test loss and train-test differences for each $\gamma$ in the semi-supervised setting.
We can see that by increasing the parameter $\gamma$, we can reduce the generalization gap.
We can see that the generalization gap tends to be smaller as we increase the value of $\gamma$.}
\end{figure*}

\begin{table}[t]
\centering
\caption{Mean IOU comparison for Pascal VOC dataset.}
\label{tab:voc_results}
\scalebox{0.87}{
\begin{tabular}{ll}
\toprule
Model                                         & mean IOU               \\ \midrule
FCN ResNet-50                                 & $0.5397 (\pm{0.0073})$ \\
FCN ResNet-50 + density-fixing ($\gamma=0.1$) & $0.5402 (\pm{0.0078})$ \\
FCN ResNet-50 + density-fixing ($\gamma=0.5$) & $0.5475 (\pm{0.0055})$ \\
FCN ResNet-50 + density-fixing ($\gamma=1.0$) & ${\bf 0.5531 (\pm{0.0061})}$ \\
FCN ResNet-50 + density-fixing ($\gamma=2.0$) & $0.5487 (\pm{0.0060})$ \\ \bottomrule
\end{tabular}
}
\end{table}

\subsection{Supervised Classification}
In this subsection, we introduce the supervised classification results, and we assumed a discrete uniform distribution for the class distribution.

Figure~\ref{fig:supervised_train_vs_test} shows the experimental results for CIFAR-10 with density-fixing regularization.
As seen in the left of this figure, baseline model and density-fixing converge at a similar speed to their best test errors.
At around $100$ epoch, a second loss reduction, Deep Double Descent~\cite{Nakkiran2020Deep}, can be observed, but this phenomenon is not disturbed by density-fixing.
From the right, we can see that by increasing the parameter $\gamma$, we can reduce the generalization gap.

Also, Table~\ref{tab:experiments_error_supervised} shows the contribution of density-fixing to the reduction of test errors.

\subsection{Semi-Supervised Classification}
In this subsection, we introduce the semi-supervised classification results.
In our experiments, we assumed a discrete uniform distribution for the class distribution and treated $1/5$ of the training data as labeled and $4/5$ of the training data as unlabeled.

Figure~\ref{fig:train_vs_test} show test loss and train-test differences for each $\gamma$ in the semi-supervised setting.
We can see that by increasing the parameter $\gamma$, it reduce the generalization gap.
In addition, CIFAR-10 and CIFAR-100, which consist of images from the same domain, have $10$ and $100$ classes, respectively, but the experimental results show that CIFAR-10 has a more significant regularization effect than CIFAR-100.
This result supports our example in Eq~\eqref{eq:uniform}.

Table~\ref{tab:experiments_error} shows a comparison of classification error for each $\gamma$.
These experimental results show that our regularization leads to improving error on the test data.

\begin{figure}[t]
    \centering
    \includegraphics[scale=0.31, bb=0 0 840 448]{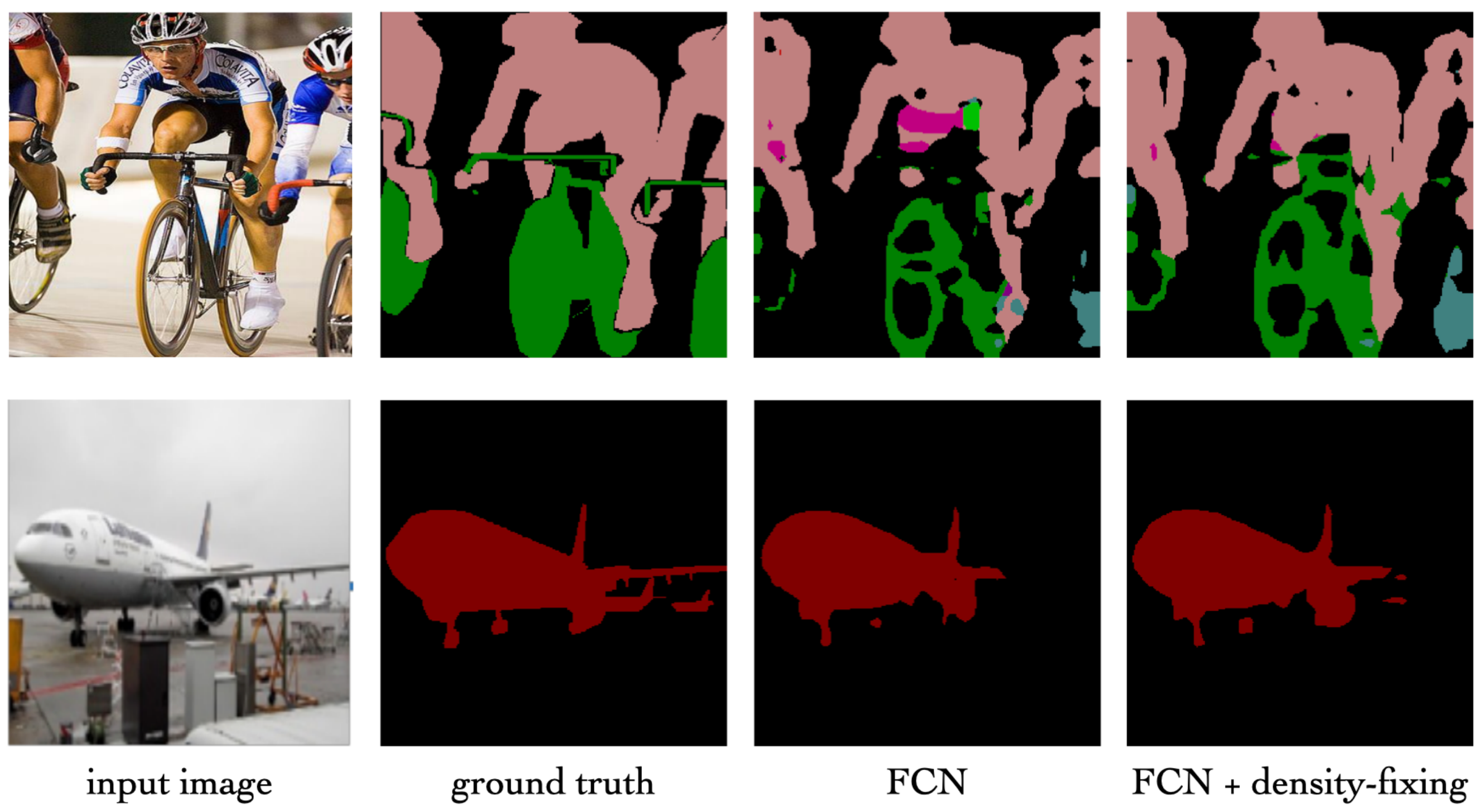}
    \caption{\label{fig:voc_segmentation} The experimental result on semantic segmentation in the Pascal VOC2012 dataset. We can see that density-fixing can improve the performance of the semantic segmentation task.}
\end{figure}

\subsection{Semantic Segmentation for the Pascal VOC Dataset}
We present the results of an experiment on semantic segmentation in the Pascal VOC2012 dataset~\cite{Everingham15}.
This dataset contains around $15,000$ labeled images belonging to $20$ categories.

We computed pixel-wise the frequency of class occurrence in the training data as a prior distribution to apply our density-fixing.
We also used FCN ResNet50~\cite{long2015fully} as a baseline model to investigate the performance difference with and without density-fixing.

Table~\ref{tab:voc_results} shows a comparison of mean IOU for each models.
We can see that our density-fixing regularization can improve the performance.

Figure~\ref{fig:voc_segmentation} shows the visualization of the semantic segmentation result.
In this experiment, we apply our density-fixing with parameter $\gamma=1$.
Recall that our density-fixing regularization has the zero-forcing property, and we can see that our regularization term contributes to eliminating the labels that never happens (e.g., the light green label in the first row which occurs in the result of FCN represents the sofa, and density-fixing eliminate it.).

\subsection{Stabilization of Generative Adversarial Networks}
Generative Adversarial Networks (GANs)~\cite{goodfellow2014generative} is one of the powerful generative model paradigms that are currently successful in various tasks, such as image generation~\cite{karras2018progressive,brock2018large} and image-to-image translation~\cite{choi2018stargan,zhu2017unpaired}.
However, GANs have the problem that their learning is very unstable.
To tackle this problem, many approaches have been proposed~\cite{miyato2018spectral,arjovsky2017wasserstein}.
We suggest that regularization by density-fixing contributes to improving the stability of GANs.
The density-fixing formulation of GANs is:
\begin{equation}
    \max_G\min_D\mathbb{E}_{\bm{x},\bm{z}}\ell(D(\bm{x}), 1) + \ell(D(G(\bm{z})), 0) + D_{KL}[p_D(y)\|q(y)],
\end{equation}
where $D$ is the discriminator, $G$ is the generator,  $\ell$ is the binary cross entropy and $q(y) = Ber(0.5)$.

Figure~\ref{fig:gan_experiment} illustrates the stabilizing effect of density-fixing the training of GANs when modeling a toy dataset (blue samples). The neural networks in these experiments are fully-connected and have three hidden layers of $512$ ReLU units.
We can see that density-fixing contributes to the stabilization of the training of GANs.

\begin{figure}[t]
    \centering
    \includegraphics[scale=0.27, bb=0 0 860 499]{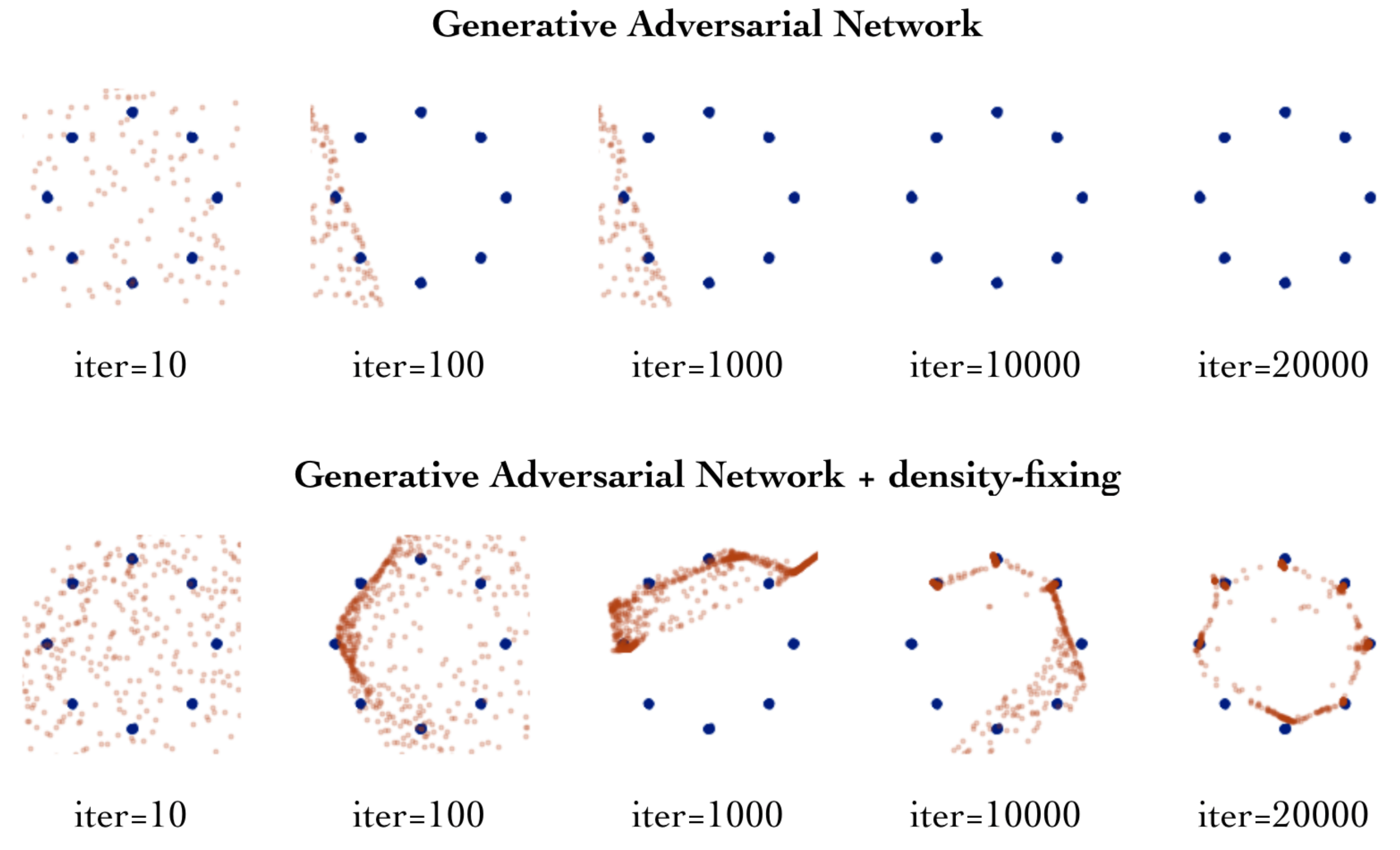}
    \caption{\label{fig:gan_experiment} Effect of density-fixing on stabilizing GANs training with $\gamma=1$.}
\end{figure}

\begin{figure}[t]
\centering
\begin{lstlisting}[language=Python]
for i, (inputs, targets) in enumerate(train_loader):
    outputs = model(inputs)
    preds = torch.softmax(outputs, 1)
    
    # density-fixing regularization
    # np.random.uniform
    p_y = uniform(0, 1, (batch_size, n_classes)
    p_y = torch.Tensor(p_y)
    p_y = torch.softmax(p_y, 1)
    R = nn.KLDivLoss()(p_y.log(), preds)
    
    # add regularization term
    loss = criterion(outputs, targets) + gamma * R
    loss.backward()
\end{lstlisting}
\caption{\label{fig:impl} Few lines of code necessary to implement density-fixing regularization in PyTorch.}
\end{figure}

\section{Conclusion and Discussion}
In this paper, we proposed a framework of regularization methods that can be used commonly for both supervised and semi-supervised learning.
Our proposed regularization method improves the generalization performance by forcing the model to approximate the prior distribution of the class.
We proved that this regularization term is naturally derived from the formula of maximum likelihood estimation.
We further investigated the asymptotic behavior of the proposed method and how the regularization terms behave when assuming a prior distribution of several classes in practice.
Our experimental results have sufficiently demonstrated the effectiveness of our proposed method.

{\small
\bibliographystyle{ieee_fullname}

\begin{thebibliography}{10}\itemsep=-1pt

\bibitem{arjovsky2017wasserstein}
Martin Arjovsky, Soumith Chintala, and L{\'e}on Bottou.
\newblock Wasserstein generative adversarial networks.
\newblock In {\em International Conference on Machine Learning}, pages
  214--223, 2017.

\bibitem{brock2018large}
Andrew Brock, Jeff Donahue, and Karen Simonyan.
\newblock Large scale gan training for high fidelity natural image synthesis.
\newblock In {\em International Conference on Learning Representations}, 2018.

\bibitem{choi2018stargan}
Yunjey Choi, Minje Choi, Munyoung Kim, Jung-Woo Ha, Sunghun Kim, and Jaegul
  Choo.
\newblock Stargan: Unified generative adversarial networks for multi-domain
  image-to-image translation.
\newblock In {\em Proceedings of the IEEE conference on computer vision and
  pattern recognition}, pages 8789--8797, 2018.

\bibitem{coates2011analysis}
Adam Coates, Andrew Ng, and Honglak Lee.
\newblock An analysis of single-layer networks in unsupervised feature
  learning.
\newblock In {\em Proceedings of the fourteenth international conference on
  artificial intelligence and statistics}, pages 215--223, 2011.

\bibitem{dietterich1995overfitting}
Tom Dietterich.
\newblock Overfitting and undercomputing in machine learning.
\newblock {\em ACM computing surveys (CSUR)}, 27(3):326--327, 1995.

\bibitem{Everingham15}
M. Everingham, S.~M.~A. Eslami, L. Van~Gool, C.~K.~I. Williams, J. Winn, and A.
  Zisserman.
\newblock The pascal visual object classes challenge: A retrospective.
\newblock {\em International Journal of Computer Vision}, 111(1):98--136, Jan.
  2015.

\bibitem{furlanello2018born}
Tommaso Furlanello, Zachary Lipton, Michael Tschannen, Laurent Itti, and Anima
  Anandkumar.
\newblock Born again neural networks.
\newblock In {\em International Conference on Machine Learning}, pages
  1607--1616, 2018.

\bibitem{goodfellow2014generative}
Ian Goodfellow, Jean Pouget-Abadie, Mehdi Mirza, Bing Xu, David Warde-Farley,
  Sherjil Ozair, Aaron Courville, and Yoshua Bengio.
\newblock Generative adversarial nets.
\newblock In {\em Advances in neural information processing systems}, pages
  2672--2680, 2014.

\bibitem{hawkins2004problem}
Douglas~M Hawkins.
\newblock The problem of overfitting.
\newblock {\em Journal of chemical information and computer sciences},
  44(1):1--12, 2004.

\bibitem{he2016deep}
Kaiming He, Xiangyu Zhang, Shaoqing Ren, and Jian Sun.
\newblock Deep residual learning for image recognition.
\newblock In {\em Proceedings of the IEEE conference on computer vision and
  pattern recognition}, pages 770--778, 2016.

\bibitem{44873}
Geoffrey Hinton, Oriol Vinyals, and Jeffrey Dean.
\newblock Distilling the knowledge in a neural network.
\newblock In {\em NIPS Deep Learning and Representation Learning Workshop},
  2015.

\bibitem{karras2018progressive}
Tero Karras, Timo Aila, Samuli Laine, and Jaakko Lehtinen.
\newblock Progressive growing of gans for improved quality, stability, and
  variation.
\newblock In {\em International Conference on Learning Representations}, 2018.

\bibitem{kimura2020mixup}
Masanari Kimura.
\newblock Mixup training as the complexity reduction.
\newblock {\em arXiv preprint arXiv:2006.06231}, 2020.

\bibitem{kimura2018anomaly}
Masanari Kimura and Takashi Yanagihara.
\newblock Anomaly detection using gans for visual inspection in noisy training
  data.
\newblock In {\em Asian Conference on Computer Vision}, pages 373--385.
  Springer, 2018.

\bibitem{kingma2014semi}
Durk~P Kingma, Shakir Mohamed, Danilo~Jimenez Rezende, and Max Welling.
\newblock Semi-supervised learning with deep generative models.
\newblock In {\em Advances in neural information processing systems}, pages
  3581--3589, 2014.

\bibitem{krizhevsky2009learning}
Alex Krizhevsky, Geoffrey Hinton, et~al.
\newblock Learning multiple layers of features from tiny images.
\newblock 2009.

\bibitem{kullback1951on}
S. Kullback and R.~A. Leibler.
\newblock On information and sufficiency.
\newblock {\em Ann. Math. Statist.}, pages 22:79--86, 1951.

\bibitem{kumar2017semi}
Abhishek Kumar, Prasanna Sattigeri, and Tom Fletcher.
\newblock Semi-supervised learning with gans: Manifold invariance with improved
  inference.
\newblock In {\em Advances in Neural Information Processing Systems}, pages
  5534--5544, 2017.

\bibitem{lawrence1997lessons}
Steve Lawrence, C~Lee Giles, and Ah~Chung Tsoi.
\newblock Lessons in neural network training: Overfitting may be harder than
  expected.
\newblock In {\em AAAI/IAAI}, pages 540--545. Citeseer, 1997.

\bibitem{lee2013pseudo}
Dong-Hyun Lee.
\newblock Pseudo-label: The simple and efficient semi-supervised learning
  method for deep neural networks.
\newblock In {\em Workshop on challenges in representation learning, ICML},
  volume~3, 2013.

\bibitem{long2015fully}
Jonathan Long, Evan Shelhamer, and Trevor Darrell.
\newblock Fully convolutional networks for semantic segmentation.
\newblock In {\em Proceedings of the IEEE conference on computer vision and
  pattern recognition}, pages 3431--3440, 2015.

\bibitem{miyato2018spectral}
Takeru Miyato, Toshiki Kataoka, Masanori Koyama, and Yuichi Yoshida.
\newblock Spectral normalization for generative adversarial networks.
\newblock In {\em International Conference on Learning Representations}, 2018.

\bibitem{Nakkiran2020Deep}
Preetum Nakkiran, Gal Kaplun, Yamini Bansal, Tristan Yang, Boaz Barak, and Ilya
  Sutskever.
\newblock Deep double descent: Where bigger models and more data hurt.
\newblock In {\em International Conference on Learning Representations}, 2020.

\bibitem{netzer2011reading}
Yuval Netzer, Tao Wang, Adam Coates, Alessandro Bissacco, Bo Wu, and Andrew~Y
  Ng.
\newblock Reading digits in natural images with unsupervised feature learning.
\newblock 2011.

\bibitem{paszke2019pytorch}
Adam Paszke, Sam Gross, Francisco Massa, Adam Lerer, James Bradbury, Gregory
  Chanan, Trevor Killeen, Zeming Lin, Natalia Gimelshein, Luca Antiga, et~al.
\newblock Pytorch: An imperative style, high-performance deep learning library.
\newblock In {\em Advances in neural information processing systems}, pages
  8026--8037, 2019.

\bibitem{roth2017stabilizing}
Kevin Roth, Aurelien Lucchi, Sebastian Nowozin, and Thomas Hofmann.
\newblock Stabilizing training of generative adversarial networks through
  regularization.
\newblock In {\em Advances in neural information processing systems}, pages
  2018--2028, 2017.

\bibitem{srivastava2014dropout}
Nitish Srivastava, Geoffrey Hinton, Alex Krizhevsky, Ilya Sutskever, and Ruslan
  Salakhutdinov.
\newblock Dropout: a simple way to prevent neural networks from overfitting.
\newblock {\em The journal of machine learning research}, 15(1):1929--1958,
  2014.

\bibitem{xie2020self}
Qizhe Xie, Minh-Thang Luong, Eduard Hovy, and Quoc~V Le.
\newblock Self-training with noisy student improves imagenet classification.
\newblock In {\em Proceedings of the IEEE/CVF Conference on Computer Vision and
  Pattern Recognition}, pages 10687--10698, 2020.

\bibitem{yun2019cutmix}
Sangdoo Yun, Dongyoon Han, Seong~Joon Oh, Sanghyuk Chun, Junsuk Choe, and
  Youngjoon Yoo.
\newblock Cutmix: Regularization strategy to train strong classifiers with
  localizable features.
\newblock In {\em Proceedings of the IEEE International Conference on Computer
  Vision}, pages 6023--6032, 2019.

\bibitem{zhang2018mixup}
Hongyi Zhang, Moustapha Cisse, Yann~N. Dauphin, and David Lopez-Paz.
\newblock mixup: Beyond empirical risk minimization.
\newblock In {\em International Conference on Learning Representations}, 2018.

\bibitem{zhu2017unpaired}
Jun-Yan Zhu, Taesung Park, Phillip Isola, and Alexei~A Efros.
\newblock Unpaired image-to-image translation using cycle-consistent
  adversarial networks.
\newblock In {\em Proceedings of the IEEE international conference on computer
  vision}, pages 2223--2232, 2017.

\bibitem{zhu2009introduction}
Xiaojin Zhu and Andrew~B Goldberg.
\newblock Introduction to semi-supervised learning.
\newblock {\em Synthesis lectures on artificial intelligence and machine
  learning}, 3(1):1--130, 2009.

\bibitem{zhu2005semi}
Xiaojin~Jerry Zhu.
\newblock Semi-supervised learning literature survey.
\newblock Technical report, University of Wisconsin-Madison Department of
  Computer Sciences, 2005.

\end{thebibliography}

}

\end{document}